\documentclass[a4paper]{article}
\usepackage{graphicx}
\usepackage{hyperref}

\usepackage{mathtools}
\usepackage{amssymb}

\usepackage{comment}
\usepackage{color}

\usepackage{amsthm}

\usepackage{stmaryrd} 
 
\usepackage{xspace}

\usepackage{authblk}
\usepackage[misc]{ifsym}
\usepackage{subcaption}
\usepackage{cleveref}
\usepackage{booktabs}

\newcommand{\p}[0]{\mathbb{P}}
\newcommand{\I}[0]{\mathbb{I}}

\newcommand{\R}[0]{\mathbb{R}}

\newcommand{\X}[0]{\mathcal{X}}

\def\roc{{\rm ROC\xspace}}
\def\auc{{\rm AUC\xspace}}
\def\argmax{\mathop{\rm arg\, max}}

\newtheorem{theorem}{Theorem}

\newtheorem{lemma}{Lemma}

\begin{document}
\title{On Tree-based Methods for Similarity Learning}
%
%
\author[1]{Stephan Cl\'emen\c{c}on}
\author[1,2]{Robin Vogel}

\affil[1]{Telecom ParisTech, LTCI, Universit\'e Paris Saclay, France\\
\href{mailto:first.last@telecom-paristech.fr}{first.last@telecom-paristech.fr}}
\affil[2]{IDEMIA, France\\ \href{mailto:first.last@idemia.fr}{first.last@idemia.fr}}
\date{}
%
%

\maketitle              
\begin{abstract}
In many situations, the choice of an adequate similarity measure or metric on
the feature space dramatically determines the performance of machine learning
methods. Building automatically such measures is the specific purpose of
metric/similarity learning. In \cite{VCB18}, similarity learning is formulated
as a pairwise bipartite ranking problem: ideally, the larger the probability
that two observations in the feature space belong to the same class (or share
the same label), the higher the similarity measure between them. From this
perspective, the $\roc$ curve is an appropriate performance criterion and it is
the goal of this article to extend recursive tree-based $\roc$ optimization
techniques in order to propose efficient similarity learning algorithms. The
validity of such iterative partitioning procedures in the pairwise setting is
established by means of results pertaining to the theory of $U$-processes and
from a practical angle, it is discussed at length how to implement them by
means of splitting rules specifically tailored to the similarity learning task.
Beyond these theoretical/methodological contributions, numerical experiments
are displayed and provide strong empirical evidence of the performance of the
algorithmic approaches we propose.

\noindent
{\bf Keywords:} Metric-Learning  $\cdot$ Rate Bound Analysis $\cdot$ Similarity
Learning $\cdot$ Tree-based Algorithms $\cdot$ $U$-processes.
\end{abstract}
\section{Introduction}
Similarity functions are ubiquitous in machine learning, they are the essential ingredient of nearest neighbor rules in classification/regression or K-means/medoids clustering methods for instance and crucially determine their performance when applied to major problems such as biometric identification or recommending system design. The goal of learning automatically from data a similarity function or a metric has been formulated in various ways, depending on the type of similarity feedback available  (\textit{e.g.} labels, preferences), see \cite{Jin2009a,Bellet2015a,Cao2016a,Jain2017a,Verma2015}. A dedicated literature has recently emerged, devoted to this class of problems that is referred to as similarity-learning or metric-learning and is now receiving much attention, see \textit{e.g.} \cite{Bellet2015c} or \cite{Kulis2012a} and the references therein. A popular framework, akin to that of multi-class classification, stipulates that pairwise similarity judgments can be directly deduced from observed class labels: a positive label is assigned to pairs formed of observations in the same class, while a negative label is assigned to those lying in different classes. In this context, similarity learning has been recently expressed as a \textit{pairwise bipartite ranking problem} in \cite{VCB18}, the task consisting in learning a similarity function that ranks the elements of a database by decreasing order of the posterior probability that they share the same label with some arbitrary query data point, as it is the case in important applications. In biometric identification (see \textit{e.g.} \cite{Jain2000a}), the identity claimed by an individual is checked by matching her biometric information, a photo or fingerprints taken at an airport for instance,
with those of authorized people gathered in a data repository of reference (\textit{e.g.} passport photos or fingerprints). Based on a given similarity function and a fixed threshold value, the elements of the database are sorted by decreasing order of similarity score with the query and those whose score exceeds the threshold specified form the collection of matching elements. The $\roc$ curve of a similarity function, \textit{i.e.} the plot of the false positive rate \textit{vs} the true positive rate as the threshold varies, appears in this situation as a natural (functional) performance measure. Whereas several approaches have been proposed to optimize a statistical counterpart of its scalar summary, the $\auc$ criterion ($\auc$ standing for Area Under the $\roc$ Curve), see \cite{McFee2010a,Huo2018a}, it is pointed out in \cite{VCB18} that more local criteria must be considered in practice: ideally, the true positive rate should be maximized under the constraint that the false positive rate remains below a fixed level, usually specified in advance on the basis of operational constraints (see \cite{Jain2000a,Jain2014a} in the case of biometric applications). If the generalization ability of solutions of empirical versions of such pointwise $\roc$ optimization problems (and the situations where fast learning rates are achievable as well) has been investigated at length in \cite{VCB18}, it is very difficult to solve in practice these constrained, generally nonconvex, optimization problems. It is precisely the goal of the present paper to address this algorithmic issue. Our approach builds on an iterative $\roc$ optimization method, referred to as {\sc TreeRank}, that has been proposed in \cite{CV09ieee} (see also \cite{CDV09} as well as \cite{CDV13} for an ensemble learning technique based on this method) and investigated at length in the standard (non pairwise) bipartite ranking setting. In this article, we establish statistical guarantees for the validity of the {\sc TreeRank} methodology, when extended to the similarity learning framework (\textit{i.e.} pairwise bipartite ranking), in the form of generalization rate bounds related to the $\sup$ norm in the $\roc$ space and discuss issues related to its practical implementation. In particular, the \textit{splitting rules} recursively implemented in the variant we propose are specifically tailored to the similarity learning task and produce symmetric tree-based scoring rules that may thus serve as similarity functions. Numerical experiments based on synthetic and real data are also presented here, providing strong empirical evidence of the relevance of this approach for similarity learning.\\
The paper is organized as follows. The rigorous formulation of similarity learning as pairwise bipartite ranking is briefly recalled in section \ref{sec:background}, together with the main principles underlying the {\sc TreeRank} algorithm for $\roc$ optimization. In section \ref{sec:main}, theoretical results proving the validity of the {\sc TreeRank} method in the pairwise setup are stated and practical implementation issues are also discussed. Section \ref{sec:exp} displays illustrative experimental results. 

\section{Background and Preliminaries}\label{sec:background}
We start with recalling key concepts of similarity learning and its natural connection with $\roc$ analysis and next briefly describe the algorithmic principles underlying the {\sc TreeRank} methodology. Throughout the article, the Dirac mass at any point $x$ is denoted by $\delta_x$, the indicator function of any event $\mathcal{E}$ by $\mathbb{I}\{\mathcal{E}\}$, and the pseudo-inverse of any cdf $F(u)$ on $\mathbb{R}$ by $F^{-1}(t)=\inf\{v\in \mathbb{R}:\; F(v)\geq t  \}$.

\subsection{Similarity Learning as Pairwise Bipartite Ranking}
We place ourselves in the probabilistic setup of multi-class classification here: $Y$ is a random label, taking its values in $\{1,\; \ldots,\; K \}$ with $K\geq 1$ say, and $X$ is a random vector defined on the same probability space, valued in a feature space $\mathcal{X}\subset \mathbb{R}^d$ with $d\geq 1$ and modelling some information hopefully useful to predict $Y$. The marginal distribution of $X$ is denoted by $\mu(dx)$, while the prior/posterior probabilities are $p_k=\mathbb{P}\{Y=k\}$ and $\eta_k(X)=\mathbb{P}\{Y=k \mid X  \}$, $k=1,\; \ldots,\; K$. The conditional distribution of the r.v. $X$ given $Y=k$ is denoted by $\mu_k$. The distribution $P$ of the generic pair $(X,Y)$ is entirely characterized by $(\mu,\; (\eta_1,\; \ldots,\; \eta_K))$. Equipped with these notations, we have $\mu=\sum_{k}p_k\mu_k$ and $p_k=\int_{\mathcal{X}}\eta_k(x)\mu(dx)$ for $k\in\{1,\;\ldots,\; K\}$.
In a nutshell, the goal pursued in this similarity learning framework is to learn from a training dataset $\mathcal{D}_n=\{(X_1,Y_1),\; \ldots,\; (X_n,Y_n)  \}$ composed of independent observations with distribution $P$ a similarity (scoring) function, that is a measurable symmetric function $s:\mathcal{X}^2\to \mathbb{R}_+$ (\textit{i.e.} $\forall (x,x')\in \mathcal{X}^2$, $s(x,x')=s(x',x)$) such that, given an independent copy $(X',Y')$ of $(X,Y)$, the larger the similarity score between the input observations $X$ and $X'$, the higher the probability that they share the same label (\textit{i.e.} that $Y=Y'$) should be. We denote by $\mathcal{S}$ the ensemble of all similarity functions.\\
\noindent {\bf Optimal rules.} Given this informal objective, the set of optimal similarity functions is obviously formed of strictly increasing transforms of the (symmetric) posterior probability $\eta(x,x')=\mathbb{P}\{Y=Y'\mid (X,X')=(x,x')\}$, namely
$$
\mathcal{S}^*=\{T\circ \eta:\;\; T:Im(\eta)\to \mathbb{R}_+\text{ borelian, strictly increasing}\},
$$
denoting by $Im(\eta)$ the support of the r.v. $\eta(X,X')=\sum_k\eta_k(X)\eta_k(X')$. A similarity function $s^*\in \mathcal{S}^*$ defines the optimal preorder\footnote{A preorder on a set $\X$ is any reflexive and transitive binary relationship on $\X$. A preorder is an order if, in addition, it is antisymmetrical.} $\preceq^*$ on the product space $\X\times \X$: for all $(x_1,x_2,x_3,x_4)\in\X^4$, $x_1$ and $x_2$ are more similar to each other than $x_3$ and $x_4$ iff $\eta(x_1, x_2)\geq \eta(x_3,x_4)$, and one then writes $(x_3,x_4)\preceq^*(x_1,x_2)$. For any query $x\in \mathcal{X}$, $s^*$ also defines a preorder $\preceq^*_{x}$ on the input space $\X$, that enables us to rank optimally all possible observations by increasing degree of similarity to $x$: for any $(x_1,x_2)\in \X^2$, $x_1$ is more similar to $x$ than $x_2$ (one writes $x_2 \preceq^*_{x} x_1$) iff $(x,x_2)\preceq^*(x,x_1)$,  that is $\eta(x,x_2)\leq \eta(x,x_1)$. 
\\
\noindent {\bf Pointwise $\roc$ curve optimization.} As highlighted in \cite{VCB18}, similarity learning can be formulated as a \textit{bipartite ranking} problem on the product space $\X\times \X$ where, given two independent realizations $(X,Y)$ and $(X',Y')$ of $P$, the input r.v. is the pair $(X,X')$ and the binary label is $Z=2\mathbb{I}\{Y=Y'  \}-1$, see \textit{e.g.} \cite{Clemencon08Ranking}. In bipartite ranking, the gold standard by which the performance of a scoring function $s$ is measured is the $\roc$ curve (see \textit{e.g.} \cite{Fawcett06} for an account of $\roc$ analysis and its applications.): one evaluates how close the preorder induced by $s$ to $\preceq^*$ is by plotting the parametric curve $t\in \mathbb{R}_+\mapsto (F_{s,-}(t),F_{s,+}(t))$,
where
\begin{equation*}
F_{s,-}(t)=\mathbb{P}\{ s(X,X')> t  \mid Z=-1 \},\; 
F_{s,+}(t)= \mathbb{P}\{ s(X,X')> t  \mid Z=+1 \},
\end{equation*}
where possible jumps are connected by line segments. This P-P plot is referred to as the $\roc$ curve of $s(x,x')$ and can be viewed as the graph of a continuous function $\alpha\in (0,1)\mapsto \roc_s(\alpha)$, where $\roc_s(\alpha)=F_{s,+}\circ F_{s,-}^{-1}(\alpha)$ at any point $\alpha\in (0,1)$ such that $F_{s,-}\circ F_{s,-}^{-1}(\alpha)=\alpha$. The curve $\roc_s$ informs us about the capacity of $s$ to discriminate between pairs with same labels and pairs with different labels: the stochastically larger  than $F_{s,-}$ the distribution $F_{s,+}$, the higher $\roc_s$. It corresponds to the type I error \textit{vs} power plot (false positive rate \textit{vs} true positive rate) of the statistical test $\mathbb{I}\{ s(X,X')> t \}$ when the null hypothesis stipulates that the labels of $X$ and $X'$ are different (\textit{i.e.} $Y\ne Y'$) and defines a partial preorder on the set $\mathcal{S}$: one says that a similarity function $s_1$ is more accurate than another one $s_2$ when, for all $\alpha\in (0,1)$, $\roc_{s_2}(\alpha)\leq \roc_{s_1}(\alpha)$. The optimality of the elements of $\mathcal{S}^*$ w.r.t. this partial preorder immediately results from a classic Neyman-Pearson argument: $\forall (s,s^*)\in \mathcal{S}\times \mathcal{S}^*$, $\roc_{s}(\alpha)\leq \roc_{s^*}(\alpha)=\roc_{\eta}(\alpha):=\roc^*(\alpha)$ for all $\alpha\in (0,1)$. For simplicity, we assume here that the conditional cdf of $\eta(X,X')$ given $Z=-1$ is invertible.
The accuracy of any $s\in \mathcal{S}$ can be measured by:
 \begin{equation}
 D_p(s,s^*)=\vert\vert \roc_s-\roc^*\vert\vert_p,
 \end{equation}
 where $s^*\in \mathcal{S}^*$ and $p\in[1,+\infty]$.
 When $p=1$, one may write $D_1(s,s^*)=\auc^*-\auc(s)$, where $\auc(s)=\int_{\alpha=0}^1\roc_s(\alpha)d\alpha$ is the \textit{Area Under the $\roc$ Curve} ($\auc$ in short) and $\auc^*=\auc(\eta)$ is the maximum $\auc$. Minimizing $D_1(s,s^*)$ boils down thus to maximizing the $\roc$ summary $\auc(s)$, whose popularity arises from its interpretation as the \textit{rate of concording pairs}:
\begin{multline*} \auc(s)=\mathbb{P}\left\{s(X_1,X_1')<s(X_2,X_2') \mid (Z_1,Z_2)=(-1,+1) \right\}\\ +\frac{1}{2}\mathbb{P}\left\{s(X_1,X_1')=s(X_2,X_2') \mid (Z_1,Z_2)=(-1,+1) \right\},
\end{multline*}
where $((X_1,X'_1),Z_1)$ and $((X_2,X_2'),Z_2)$ denote independent copies of $((X,X'),Z)$.
A simple empirical counterpart of $\auc(s)$  can be derived from this formula, paving the way for the implementation of "empirical risk minimization" strategies, see \cite{Clemencon08Ranking} (the algorithms proposed to optimize the $\auc$ criterion or surrogate performance measures are too numerous to be listed exhaustively here). However, as mentioned precedingly, in many applications, one is interested in finding a similarity function that optimizes the $\roc$ curve at specific points $\alpha\in (0,1)$.
The superlevel sets of similarity functions in $\mathcal{S}^*$ define the solutions of pointwise $\roc$ optimization problems in this context.
In the above framework, it indeed follows from Neyman Pearson's lemma that the test statistic of type I error less than $\alpha$ with maximum power is the indicator function of the set $\mathcal{R}^*_{\alpha}=\{(x,x')\in \X^2:\; \eta(x,x')\geq Q^*_{\alpha}  \}$, where $Q^*_{\alpha} $ is the conditional quantile of the r.v. $\eta(X,X')$ given $Z=-1$ at level $1-\alpha$. Considering similarity functions that are bounded by $1$ only, it corresponds to the unique solution of the problem:
\begin{equation*}
\max_{{\scriptsize \begin{array}{c}
s:\X^2\rightarrow [0,1],\\
\text{ borelian}
\end{array}}
} \mathbb{E}[s(X,X') \mid Z=+1] \quad \text{subject to} \quad \mathbb{E}[s(X,X') \mid Z=-1 ] \leq \alpha.
\end{equation*}
Though its formulation is natural, this constrained optimization problem is very difficult to solve in practice, as discussed at length in \cite{VCB18}. This suggests the extension to the similarity ranking framework of the {\sc TreeRank} approach for $\roc$ optimization (see \cite{CV09ieee} and \cite{CDV09}), recalled below. Indeed, in the standard (non pairwise) statistical learning setup for bipartite ranking, whose probabilistic framework is the same as that of binary classification and stipulates that training data are i.i.d. labeled observations, this recursive technique builds (piecewise constant) scoring functions $s$, whose accuracy can be guaranteed in terms of $\sup$ norm (\textit{i.e.} for which $D_{\infty}(s,s^*)$ can be controlled) and it is the essential purpose of the subsequent analysis to prove that this remains true when the training observations are of the form $\{((X_i,X_j),\; Z_{i,j}):\; 1\leq i<j\leq n \}$, where $Z_{i,j}=2\mathbb{I}\{Y_i=Y_j  \}-1$ for $1\leq i<j\leq n$, and are thus far from being independent. Regarding its implementation, attention should be paid to the fact that the splitting rules for recursive partitioning of the space $\mathcal{X}\times \mathcal{X}$ must ensure that the decision functions produced by the algorithm fulfill the symmetric property.

\subsection{Recursive $\roc$ Curve Optimization - The TreeRank Algorithm}\label{subsec:treerank}
Because they offer a visual model summary in the form of an easily interpretable binary tree graph, decision trees remain very popular among practicioners, see \textit{e.g.} \cite{cart84} or \cite{Quinlan}. In general, predictions are computed through a hierarchical combination of elementary rules comparing the value taken by a (quantitative) component of the input information (the \textit{split variable}) to a certain threshold (the \textit{split value}). In contrast to (supervised) learning problems such as classification/regression, which are of local nature, predictive rules for a global problem such as \textit{similarity learning} cannot be described by a simple (tree-structured) partition of $\mathcal{X}\times \mathcal{X}$: the (symmetric) cells corresponding to the terminal leaves of the binary decision tree must be sorted in order to define a similarity function. 

\noindent {\bf Similarity Trees.} We define a \textit{similarity tree} as a binary tree whose leaves all correspond to symmetric subsets $\mathcal{C}$ of the product space $\mathcal{X}\times \mathcal{X}$ (\textit{i.e.} $\forall(x,x')\in\mathcal{X}^2$; $(x,x')\in\mathcal{C}\Leftrightarrow (x',x)\in \mathcal{C}$) and is equipped with a 'left-to-right' orientation, that defines a tree-structured collection of similarity functions. Incidentally, the symmetry property makes it a specific \textit{ranking tree}, using the terminology introduced in \cite{CV09ieee}.
The root node of a  tree $\mathcal{T}_J$ of depth $J\geq 0$ corresponds to the whole space $\mathcal{X}\times \mathcal{X}$: $\mathcal{C}_{0,0}=\mathcal{X}^2$, while each internal node $(j,k)$ with $j<J$ and $k\in \{0,\; \ldots,\; 2^j -1 \}$ represents a subset $\mathcal{C}_{j,k}\subset\X^2$, whose left and right siblings respectively correspond to (symmetric) disjoint subsets $\mathcal{C}_{j+1, 2k}$ and $\mathcal{C}_{j+1, 2k+1}$ such that $\mathcal{C}_{j,k}=\mathcal{C}_{j+1, 2k}\cup \mathcal{C}_{j+1, 2k+1}$. Equipped with the left-to-right orientation, any subtree $\mathcal{T}\subset \mathcal{T}_J$ defines a preorder on $\mathcal{X}^2$: the degree of similarity being the same for all pairs $(x,x')$ lying in the same terminal cell of $\mathcal{T}$. The similarity function related to the oriented tree $\mathcal{T}$ can be written as:
\begin{equation*}\label{eq:score_tree}
\forall (x,x')\in \mathcal{X}^2,\;\; s_{\mathcal{T}}(x,x')=\sum_{\mathcal{C}_{j,k}:\text{ terminal leaf of }\mathcal{T}}2^J\left (1-\frac{k}{2^j}\right)\cdot \mathbb{I}\{(x,x')\in \mathcal{C}_{j,k}  \}.
\end{equation*}
Observe that its symmetry results from that of the $\mathcal{C}_{j,k}$'s. 
 \ 
The $\roc$ curve of the similarity function $s_{\mathcal{T}}(x,x')$ is the piecewise linear curve connecting the knots:
 $$
 (0,0)\text{ and }\; \left(\sum_{l=0}^kF_-(\mathcal{C}_{j,l}),\; \sum_{l=0}^kF_+(\mathcal{C}_{j,l})\right)
 \text{ for all terminal leaf } \mathcal{C}_{j,k}\text{ of }\mathcal{T},$$
 denoting by $F_{\sigma}$ the conditional distribution of $(X,X')$ given $Z=\sigma1$, $\sigma\in\{-.\; +\}$. Setting $p_+=\mathbb{P}\{Z=+1\}=\sum_kp_k^2$, we have $F_+=(1/p_+)\sum_kp_k^2\cdot \mu_k\otimes \mu_k$ and $F_-=(1/(1-p_+))\sum_{k\neq l}p_kp_l\cdot \mu_k\otimes\mu_l$.
 A statistical version can be computed by replacing the $F_{\sigma}(\mathcal{C}_{j,l})$'s by their empirical counterpart. 
\\
\noindent {\bf Growing the Similarity Tree.}
The {\sc TreeRank} algorithm, a bipartite ranking technique optimizing the $\roc$ curve in a recursive fashion, has been introduced in \cite{CV09ieee} and its properties have been investigated in \cite{CDV09} at length. Its output consists of a tree-structured scoring rule \eqref{eq:score_tree} with a $\roc$ curve that can be viewed as a piecewise linear approximation of $\roc^*$ obtained by a Finite Element Method with implicit scheme and is proved to be nearly optimal in the $D_1$ sense under mild assumptions. The growing stage is performed as follows. At the root, one starts with a constant similarity function $s_1(x,x')= \mathbb{I}\{(x,x')\in \mathcal{C}_{0,0}\}\equiv 1$ and after $m=2^j+k$ iterations, $0\leq k<2^j$, the current similarity function is $$s_m(x,x')=\sum_{l=0}^{2k-1}(m-l)\cdot \mathbb{I}\{(x,x')\in \mathcal{C}_{j+1,l}\}+\sum_{l=k}^{2^j-1}(m-k-l)\cdot \mathbb{I}\{(x,x')\in \mathcal{C}_{j,l}\}$$ and the cell $\mathcal{C}_{j,k}$ is split so as to form a refined version of the similarity function,
$$s_{m+1}(x,x')=\sum_{l=0}^{2k}(m-l)\cdot \mathbb{I}\{(x,x')\in \mathcal{C}_{j+1,l}\}+\sum_{l=k+1}^{2^j-1}(m-k-l)\cdot \mathbb{I}\{(x,x')\in \mathcal{C}_{j,l}\}$$ namely, with maximum (empirical) $\auc$. Therefore, it happens that this problem boils down to solve a cost-sensitive binary classification problem on the set $\mathcal{C}_{j,k}$, see subsection 3.3 in \cite{CDV09}. Indeed, one may write the $\auc$ increment as
$$\auc(s_{m+1})-\auc(s_m)=\frac{1}{2}F_-(\mathcal{C}_{j,k})F_+(\mathcal{C}_{j,k})
\times (1-\Lambda(\mathcal{C}_{j+1,2k}\mid \mathcal{C}_{j,k})),$$
$$\text{where }\Lambda(\mathcal{C}_{j+1,2k}\mid \mathcal{C}_{j,k})\overset{def}{=}F_+(\mathcal{C}_{j,k}\setminus \mathcal{C}_{j+1,2k})/F_+(\mathcal{C}_{j,k})+F_-(\mathcal{C}_{j+1,2k})/F_-(\mathcal{C}_{j,k}).$$ Setting $p=F_+(\mathcal{C}_{j,k})/(F_-(\mathcal{C}_{j,k})+F_+(\mathcal{C}_{j,k}))$, the crucial point of the {\sc TreeRank} approach is that the quantity $2p(1-p)\Lambda(\mathcal{C}_{j+1,2k}\mid \mathcal{C}_{j,k})$ can be interpreted as the cost-sensitive error of a classifier on $\mathcal{C}_{j,k}$ predicting positive label for any pair lying in $\mathcal{C}_{j+1,2k}$ and negative label fo all pairs in $\mathcal{C}_{j,k}\setminus\mathcal{C}_{j+1,2k}$ with cost $p$ (respectively, $1-p$) assigned to the error consisting in predicting label $+1$ given $Z=-1$ (resp., label $-1$ given $Z=+1$), balancing thus the two types of error. Hence, at each iteration of the similarity tree growing stage, the {\sc TreeRank} algorithm calls a \textit{cost-sensitive} binary classification algorithm, termed {\sc LeafRank}, in order to solve a statistical version of the problem above (replacing the theoretical probabilities involved by their empirical counterparts) and split $\mathcal{C}_{j,k}$ into $\mathcal{C}_{j+1,2k}$ and $\mathcal{C}_{j+1,2k+1}$. As described at length in \cite{CDV09}, one may use cost-sensitive versions of celebrated binary classification algorithms such as {\sc CART} or {\sc SVM} for instance as {\sc LeafRank} procedure, the performance depending on their ability to capture the geometry of the level sets $\mathcal{R}^*_{\alpha}$ of the posterior probability $\eta(x,x')$. As highlighted above, in order to apply the {\sc TreeRank} approach to similarity learning, a crucial feature the {\sc LeafRank} procedure implemented must have is the capacity to split a region in subsets that are both stable under the reflection $(x,x')\in \mathcal{X}^2\mapsto (x',x)$. This point is discussed in the next section. Rate bounds  for the {\sc TreeRank} method in the $\sup$ norm sense are also established therein in the statistical framework of similarity learning, when the set of training examples $\{((X_i,X_j),\; Z_{i,j} \}_{i<j}$ is composed of non independent observations with binary labels, formed from the original multi-class classification dataset $\mathcal{D}_n$.



\section{A Tree-Based Approach to Similarity Learning}\label{sec:main}
We now investigate how the {\sc TreeRank} method for $\roc$ optimization recalled in the preceding section can be extended to the framework of similarity-learning and next establish learning rates in $\sup$ norm in this context. 

\subsection{A Similarity-Learning Version of {\sc TreeRank}}
From a statistical perspective, a learning algorithm can be derived from the recursive approximation procedure recalled in the previous section, simply by replacing the quantities $F_{\sigma}(\mathcal{C})$, $\sigma\in\{-,\; +\}$ and $\mathcal{C}\subset \mathcal{X}\times \mathcal{X}$ borelian, by their empirical counterparts based on the dataset $\mathcal{D}_n$:
 \begin{equation}\label{eq:ratio}
 \widehat{F}_{\sigma,n}(\mathcal{C})=\frac{1}{n_{\sigma}}\sum_{i<j}\mathbb{I}\{(X_i,X_j)\in \mathcal{C},\; Z_{i,j}=\sigma 1  \},
 \end{equation}
 with $n_{\sigma}=(2/(n(n-1)))\sum_{i<j}\mathbb{I}\{Z_{i,j}=\sigma 1  \}$.
 Observe incidentally that the quantities \eqref{eq:ratio} are by no means i.i.d. averages, but take the form of ratios of $U$-statistics of degree two (\textit{i.e.} averages over pairs of observations, \textit{cf} \cite{Lee90}), see section 3 in \cite{VCB18}.  For this reason, a specific rate bound analysis (ignoring bias issues) guaranteeing the accuracy of the {\sc TreeRank} approach in the similarity learning framework is carried out in the next subsection. 

 \begin{center}
 \fbox{
 \begin{minipage}[t]{11.5cm}
 \medskip
  {\small
 \begin{center}
 {\sc The Similarity TreeRank Algorithm}
 \end{center}
 

 {\bf Input.} Maximal depth $D\geq 1$ of the similarity tree, class $\mathcal{A}$ of measurable and symmetric subsets of $\mathcal{X}\times \mathcal{X}$, training dataset $\mathcal{D}_n=\{(X_1,Y_1),\; \ldots,\; (X_n,Y_n)  \}$.
 
 \begin{enumerate}
 \item ({\sc Initialization.}) Set $\mathcal{C}_{0,0}=\mathcal{X}\times \mathcal{X}$, $\alpha_{d,0}=\beta_{d,0}=0$ and $\alpha_{d,2^d}=\beta_{d,2^d}=1$
 for all $d\ge 0$.
 
 
 \item ({\sc Iterations.}) For $d=0,\;\ldots,\;D-1$ and $k=0,\;\ldots,\; 2^{d}-1$:
 
 
 \begin{enumerate}
 \item ({\sc Optimization step.}) Set the entropic measure:
 \begin{eqnarray*}
 \Lambda_{d, k+1}(\mathcal{C})& =& (\alpha_{d,k+1}-\alpha_{d,k}) \widehat{F}_{+,n}(\mathcal{C})
 -(\beta_{d,k+1}-\beta_{d,k})\widehat{F}_{-,n}(\mathcal{C})~.
 \end{eqnarray*}
 Find the best subset $\mathcal{C}_{d+1,2k}$ of the cell $\mathcal{C}_{d,k}$ in the $\auc$ sense:
 \begin{equation}\label{eq:cost_classif}
 \mathcal{C}_{d+1,2k}=\argmax_{\mathcal{C}\in\mathcal{A}, ~\mathcal{C}\subset \mathcal{C}_{d,k}}
 \widehat{\Lambda}_{d, k+1}(\mathcal{C})~.
 \end{equation} 
 Then, set $\mathcal{C}_{d+1,2k+1}=\mathcal{C}_{d,k} \setminus \mathcal{C}_{d+1,2k}$.
 
 \item ({\sc Update.}) Set
 \begin{equation*}
 \alpha_{d+1,2k+1}= \alpha_{d,k}+\widehat{F}_{-,n}(\mathcal{C}_{d+1,2k}),\;
 \beta_{d+1,2k+1}=\beta_{d,k} +\widehat{F}_{+,n}(\mathcal{C}_{d+1,2k})
 \end{equation*}
 \begin{equation*}
\text{and } \alpha_{d+1,2k+2} = \alpha_{d,k+1},\; 
 \beta_{d+1,2k+2} = \beta_{d,k+1}~.
 \end{equation*}
 
 \end{enumerate}
 
 
 \item ({\sc Output.}) After $D$ iterations, get the piecewise constant similarity function:
 \begin{equation}\label{eq:sim_funct}
 s_D(x,x')=\sum_{k=0}^{2^{D}-1} (2^{D}-k) ~\I\{(x,x')\in \mathcal{C}_{D,k}\},
 \end{equation}
 together with an estimate of the curve $\roc(s_D,.)$, namely the broken line $\widehat{\roc}(s_D,.)$ that connects the knots $\{(\alpha_{D,k}, \beta_{D,k}):\; k=0,\ldots, 2^D \}$, and
  the following estimate of $\auc(s_D)$:
 \begin{eqnarray*}
 \widehat{\auc}(s_D)&=&\int_{\alpha=0}^1 \widehat{\roc}(s_D,\alpha)d\alpha
 =\frac{1}{2}+\frac{1}{2}\sum_{k=0}^{2^{D-1}-1}\widehat{\Lambda}_{D-1,k+1}(\mathcal{C}_{D,2k}).
 \end{eqnarray*}
 \end{enumerate}
 }
 \end{minipage}
 }
 \end{center}
The symmetry property of the function \eqref{eq:sim_funct} output by the learning algorithm is directly inherited from that of the candidate subsets $\mathcal{C}\in\mathcal{A}$ of the product space $\mathcal{X}\times\mathcal{X}$ among which the $\mathcal{C}_{d,k}$'s are selected. We new explain at length how to perform the optimization step \eqref{eq:cost_classif} in practice in the similarity learning context. 
\\
\noindent{\bf Splitting for Similarity Learning.} As recalled in subsection \ref{subsec:treerank}, solving \eqref{eq:cost_classif} boils down to finding the best classifier on $\mathcal{C}_{d,k}\subset \mathcal{X}^2$ of the form 
$$g_{\mathcal{C}\mid \mathcal{C}_{d,k}}(x,x')=\mathbb{I}\{(x,x')\in \mathcal{C} \}-\mathbb{I}\{(x,x')\in \mathcal{C}\setminus \mathcal{C}_{d,k}  \}\; \text{ with } \mathcal{C}\subset \mathcal{C}_{d,k},\;\; \mathcal{C}\in\mathcal{A},$$ in the empirical $\auc$ sense, that is to say that minimizing a statistical version of the cost-sensitive classification error based on $\{((X_i,X_j),\; Z_{i,j}):\;\; 1\leq i<j\leq n,\; (X_i,X_j)\in  \mathcal{C}_{d,k}  \}$
$$
\Lambda( \mathcal{C} \mid \mathcal{C}_{d,k} )=\frac{\mathbb{P}\{g_{\mathcal{C}\mid \mathcal{C}_{d,k}}(X,X')=1 \mid Z=-1\}}{\mathbb{P}\{(X,X')\in \mathcal{C}_{d,k} \mid Z=-1\}}+\frac{\mathbb{P}\{g_{\mathcal{C}\mid \mathcal{C}_{d,k}}(X,X')=-1 \mid Z=1\}}{\mathbb{P}\{(X,X')\in \mathcal{C}_{d,k} \mid Z=1\}}.
$$
Notice that, equipped with the notations previously introduced, the statistical version of $\Lambda( \mathcal{C} \mid \mathcal{C}_{d,k} )$ is 
$\Lambda_{d, k+1}(\mathcal{C})/\left((\alpha_{d,k+1}-\alpha_{d,k})(\beta_{d,k+1}-\beta_{d,k})\right)$.
In \cite{CDV09}, it is highlighted that, in the standard ranking bipartite setup, any (cost-sensitive) classification algorithm (\textit{e.g.} Neural Networks, {\sc CART}, {\sc Random Forest}, SVM, nearest neighbours) can be possibly used for splitting, whereas, in the present framework, classifiers are defined on product spaces and the symmetry issue must be addressed.  
For simplicity, assume that $\mathcal{X}$ is a subset of the space $\mathbb{R}^q$, $q\geq 1$, whose canonical basis is denoted by $(e_1,\; \ldots,\; e_q)$. Denote by $P_{V}(x,x')$ the orthogonal projection of any point $(x,x')$ in $\mathbb{R}^q\times \mathbb{R}^q$ equipped with its usual Euclidean structure onto the subspace $V=Span((e_1,\; e_1),\; \ldots,\; (e_q,\; e_q) )$. Let $W$ be $V$'s orthogonal complement in $\mathbb{R}^q\times \mathbb{R}^q$. For any $(x,x')\in \mathcal{X}^2$, denote by $f(x,x')=(f_1(x,x'),\; \ldots,\; f_{2q}(x,x'))$ the $2q$-dimensional vector, whose first $q$ components are the coordinates of the projection $P_{V}(x,x')$ of $(x,x')$ onto the subspace $V$ in an orthonormal basis of $V$ (say $\{(1/\sqrt{2})(e_1,\; e_1),\; \ldots,\; (1/\sqrt{2})(e_q,\; e_q)  \}$ for instance) and whose last components 
are formed by the \textit{absolute values} of the coordinates of the projection $P_{W}(x,x')$ of $(x,x')$ onto $W$ expressed in a given orthonormal basis (say $\{(1/\sqrt{2})(e_1,\; -e_1),\; \ldots,\; (1/\sqrt{2})(e_q,\; -e_q)  \}$ for instance). Observing that, by construction,
$
f(x,x')=f(x',x) \text{ for all } (x,x')\in \mathcal{X}^2$, our proposal relies on the following result (whose proof is straightforward and left to the reader).
\begin{lemma} 
    Let $S:\mathcal{X}^2\rightarrow \mathbb{R}$. Then, $S$ is symmetric iff there exists $s:\mathbb{R}^{q}\times \mathbb{R}_+^q\rightarrow \mathbb{R}$ such that:
$
\forall (x,x')\in \mathcal{X}^2,\;\; S(x,x')=(s\circ f )(x,x')$.
\label{lemma:sym_transform}
\end{lemma}

In order to get splits that are symmetric w.r.t. the reflection $(x,x')\mapsto (x',x)$, we propose to build directly classifiers of the form $(G\circ f)(x,x')$. In practice, this splitting procedure referred to as {\sc Symmetric LeafRank} and summarized below simply consists in using as input space $\mathbb{R}^{q}\times \mathbb{R}_+^q$ rather than $\mathbb{R}^{2q}$ and considering as training labeled observations the dataset $\{(f(X_i,X_j),\; Z_{i,j}):\;\; 1\leq i<j\leq n,\; (X_i,X_j)\in  \mathcal{C}_{d,k}  \}$ when running a cost-sensitive classification algorithm.
Just like in the original version of the {\sc TreeRank} method, the growing stage can be followed by a pruning procedure, where children of a same parent node are recursively merged in order to produce a similarity subtree that maximizes an estimate of the $\auc$ criterion, based on cross-validation usually, one may refer to section 4 in \cite{CDV09} for further details. In addition, as in the standard bipartite ranking context, the {\sc Ranking Forest} approach (see \cite{CDV13}), an \textit{ensemble learning} technique based on {\sc TreeRank} that combines aggregation and randomization, can be implemented to dramatically improve stability and accuracy of similarity tree models both at the same time, while preserving their advantages (\textit{e.g.} scalability, interpretability).

\begin{center}
\fbox{
\begin{minipage}[t]{11.5cm}{
\medskip
{\small
\begin{center}
{\sc Symmetric LeafRank}
\end{center}

\begin{itemize}
\item {\bf Input.} Pairs $\{((X_i,X_j),\; Z_{i,j}):\; 1\leq i<j\leq n,\; (X_i,X_j)\in \mathcal{C}_{d,k}  \}$ lying in the (symmetric) region to be split. Classification algorithm $\mathcal{A}$.
\item {\bf Cost.} Compute the number of positive pairs lying in the region $\mathcal{C}_{d,k}$ 
$$
p=\frac{\sum_{1\leq i<j\leq n} \mathbb{I}\{ (X_i,X_j)\in \mathcal{C}_{d,k},\; Z_{i,j}=+1 \}}{\sum_{1\leq i<j\leq n} \mathbb{I}\{ (X_i,X_j)\in \mathcal{C}_{d,k} \}}
$$  
\item {\bf Cost-sensitive classification.} Based on the labeled observations
$$
\left\{\left(f(X_i,X_j),\; Z_{i,j} \right):\; 1\leq i<j\leq n,\; (X_i,X_j)\in \mathcal{C}_{d,k}   \right\},
$$
run algorithm $\mathcal{A}$ with cost $p$ for the false positive error and cost $1-p$ for the false negative error to produce a (symmetric) classifier $g(x,x')$ on $\mathcal{C}_{d,k}$. 

\item {\bf Output} Define the subregions:
$$\mathcal{C}_{d+1,2k}=\{(x,x')\in \mathcal{C}_{d,k}:\; g(x,x')=+1  \} \text{ and } \mathcal{C}_{d+1,2k+1}=\mathcal{C}_{d,k}\setminus \mathcal{C}_{d+1,2k}.$$
\end{itemize}
}

}

\end{minipage}}

\end{center}
\label{subsection_3-1}

\subsection{Generalization Ability - Rate Bound Analysis}

We now prove that the theoretical guarantees formulated in the $\roc$ space equipped with the $\sup$ norm  that have been established for the {\sc TreeRank} algorithm in the standard bipartite ranking setup in \cite{CV09ieee} remain valid in the similarity learning framework. The rate bound result stated below is the analogue of Corollary 1 in \cite{CV09ieee}. The following technical assumptions are involved:
\begin{itemize}
\item the feature space $\mathcal{X}$ is bounded;
\item $\alpha\mapsto \roc^*(\alpha)$ is twice differentiable with a bounded first order derivative;
\item the class $\mathcal{A}$ is intersection stable, \textit{i.e.} $\forall (\mathcal{C},\; \mathcal{C}')\in \mathcal{A}^2$, $\mathcal{C}\cap \mathcal{C}'\in \mathcal{A}$;
\item the class $\mathcal{A}$ has finite {\sc VC} dimension $V<+\infty$;
\item we have $\{(x,x')\in \mathcal{X}^2:\; \eta(x,x')\geq q  \}\in \mathcal{A}$ for any $q\in [0,1]$;
\end{itemize}

\begin{theorem} \label{thm:rate} Assume that the conditions above are fulfilled. 
Choose $D=D_n$ so that $D_n\sim \sqrt{\log n}$, as $n\rightarrow \infty$, and let $s_{D_n}$ denote the output of the {\sc Similarity TreeRank} algorithm. Then, for all $\delta >0$, there exists a constant $\lambda$ s.t., with probability at least $1-\delta$, we have for all $n\geq 2$:
$
D_{\infty}(s_{D_n},s^*)\leq \exp(-\lambda \sqrt{\log n})$.
\end{theorem}
\begin{proof}
The proof is based on the following lemma, proved in \cite{VCB18} (in a more general version, the present one being a restriction to classes of indicator functions), which provides upper confidence bounds for the suprema of collections of ratios of $U$-statistics.
\begin{lemma} (Lemma 1, \cite{VCB18}) Suppose that Theorem \ref{thm:rate}'s assumptions are fulfilled. Let $\sigma\in\{-,\; + \}$. For any $\delta\in (0,1)$, we have with probability at least $1-\delta$,
$$
\sup_{\mathcal{C}}\left\vert \widehat{F}_{\sigma,n}(\mathcal{C})- F_{\sigma}(\mathcal{C}) \right\vert \leq  2C\sqrt{\frac{V}{n}}+2 \sqrt{\frac{\log (1/\delta)}{n-1}},
$$
where $C$ is a universal constant, explicited in \cite{Bousquet2004} (see page 198 therein).
\end{lemma}
This crucial result permits to control the deviation of the progressive outputs of the {\sc Similarity TreeRank} algorithm and those of the nonlinear approximation scheme (based on the true quantities) investigated in \cite{CV09ieee}. The proof can be thus derived by following line by line the argument of Corollary 1 in \cite{CV09ieee}.
$\square$
\end{proof}

This \textit{universal} logarithmic rate bound may appear slow at first glance but attention should be paid to the fact that it directly results from the hierarchical structure of the partition induced by the tree construction and the \textit{global} nature of the similarity learning problem.
As pointed out in \cite{CV09ieee} (see Remark 14 therein), the same rate bound holds true for the deviation in $\sup$ norm between the empirical $\roc$ curve $\widehat{\roc}(s_{D_n},.)$ output by the {\sc TreeRank} algorithm and the optimal curve $\roc^*$.


\section{Illustrative Numerical Experiments}\label{sec:exp}

\label{exp_sim_data}

To begin with, we study the ability of similarity ranking trees to retrieve the
optimal ROC curve for synthetic data, issued from a random tree of depth
$D_{gt}$ with a noise parameter $\delta$. Our experiments illustrate three
aspects of learning a similarity $s_D$ with TreeRank of depth $D$: the impact
of the class asymmetry $p_+ \ll 1-p_+$ as seen in the bounds of \cite{VCB18},
the trade-off between number of training instances and model complexity, see 
\cref{thm:rate}, and finally the impact of model biais.
Results are summarized in \cref{tab:simulated_data_experiments}. 
Details about the synthetic data experiments and real data experiments can be found
in the appendix.

\captionsetup[table]{skip=10pt}
\begin{table}
    \centering
    \noindent\makebox[\textwidth]{
    \scriptsize
    \begin{tabular}[t]{cllcllcll}
    \toprule
    \multicolumn{3}{c}{Class asymmetry} & 
    \multicolumn{3}{c}{Model complexity} & 
    \multicolumn{3}{c}{Model bias} \\
    \cmidrule(lr){1-3} \cmidrule(lr){4-6} \cmidrule(lr){7-9}
    $p_+$ & $D_1(s_{D}, s^*)$ & $D_\infty(s_{D}, s^*)$ & 
    $D_{\text{gt}}$ & $D_1(s_{D}, s^*)$ & $D_\infty(s_{D}, s^*)$ & 
    $D$ & $D_1(s_{D}, s^*)$ & $D_\infty(s_{D}, s^*)$ \\ 
    \cmidrule(lr){1-3} \cmidrule(lr){4-6} \cmidrule(lr){7-9}
    $0.5$ & $ 0.07 (\pm 0.07)$ & $ 0.30 (\pm 0.07)$ & 
     $1$ &  $0.00 (\pm 0.01)$ &  $0.06 (\pm 0.01)$ & 
     $1$ &  $0.21 (\pm 0.13)$ &  $0.65 (\pm 0.13)$ \\ 

     $10^{-1}$ &  $ 0.08 (\pm 0.08)$ & $ 0.31 (\pm 0.08)$ & 
     $2$ &  $0.03 (\pm 0.04)$ &  $0.20 (\pm 0.04)$ & 
     $2$ &  $0.11 (\pm 0.10)$ &  $0.43 (\pm 0.10)$ \\ 

     $10^{-3}$ &  $ 0.42 (\pm 0.17)$ & $ 0.75 (\pm 0.17)$ & 
     $3$ &  $0.07 (\pm 0.07)$ &  $0.30 (\pm 0.07)$ & 
     $3$ &  $0.07 (\pm 0.07)$ &  $0.30 (\pm 0.07)$ \\ 

     $2\cdot10^{-4}$ &  $ 0.45 (\pm 0.08)$ & $ 0.81 (\pm 0.08)$ & 
     $4$ &  $0.12 (\pm 0.09)$ &  $0.43 (\pm 0.09)$ & 
     $8$ &  $0.06 (\pm 0.06)$ &  $0.28 (\pm 0.06)$ \\ 
    \cmidrule(lr){1-3} \cmidrule(lr){4-6} \cmidrule(lr){7-9}
    \multicolumn{3}{l}{Parameters: $D= D_{gt} = 3$.} & 
    \multicolumn{3}{c}{$D_{gt} = D$, $p=0.5$.} & 
    \multicolumn{3}{c}{$D_{gt} = 3$, $p=0.5$.} \\ 
    \midrule
    \multicolumn{9}{l}{
	Shared parameters: $\X = \R^3$, $\delta = 0.01$, $n_{\text{test}}=100,000$,
	$n_{\text{train}} = 150 \cdot (5/4)^{D_{gt}^2}$.
    } \\
    \bottomrule
    \end{tabular}
}
    \hfill
    \caption{Synthetic data experiments.
    Between parenthesis are 95\%-confidence intervals 
    based off the normal approximation obtained on 400 runs.
    }
    \label{tab:simulated_data_experiments}
\end{table}

\section{Conclusion}
In situations where multi-class data are available, the objective of
\textit{similarity learning} can be naturally formulated as a $\roc$ curve
optimization problem, whose solutions are given by similarity functions
yielding a maximal true positive rate with a false positive rate below a fixed
value of reference, when thresholded at an appropriate level. Given the
importance of this learning task, that finds its motivation in many practical
problems, related to biometrics applications in particular, the present paper
proposes an extension of the recursive approach {\sc TreeRank} for $\roc$
optimization to the similarity framework. Precisely, from an algorithmic
viewpoint, it is shown how to adapt it in order to build \textit{symmetric}
scoring functions and, from a theoretical angle, the accuracy properties are
proved to be preserved in spite of the complexity of the data functional that
is optimized by the algorithm in a recursive manner. Experimental results
supporting the approach promoted are also presented.

%
%
%
 \bibliographystyle{abbrv}
 \bibliography{Ref_Ranking}

\begin{thebibliography}{10}

\bibitem{Bellet2015a}
A.~Bellet and A.~Habrard.
\newblock {R}obustness and {G}eneralization for {M}etric {L}earning.
\newblock {\em {N}eurocomputing}, 151(1):259--267, 2015.

\bibitem{Bellet2015c}
A.~Bellet, A.~Habrard, and M.~Sebban.
\newblock {\em {M}etric {L}earning}.
\newblock {M}organ \& {C}laypool {P}ublishers, 2015.

\bibitem{Bousquet2004}
O.~Bousquet, S.~Boucheron, and G.~Lugosi.
\newblock Introduction to statistical learning theory.
\newblock In {\em Advanced Lectures on Machine Learning}, pages 169--207. 2004.

\bibitem{cart84}
L.~Breiman, J.~Friedman, R.~Olshen, and C.~Stone.
\newblock {\em {Classification and Regression Trees}}.
\newblock Wadsworth and Brooks, 1984.

\bibitem{Cao2016a}
Q.~Cao, Z.-C. Guo, and Y.~Ying.
\newblock {G}eneralization {B}ounds for {M}etric and {S}imilarity {L}earning.
\newblock {\em {M}achine {L}earning}, 102(1):115--132, 2016.

\bibitem{CDV13}
G.~Cl\'emen\c{c}on, M.~Depecker, and N.~Vayatis.
\newblock {Ranking Forests}.
\newblock {\em J. Mach. Learn. Res.}, 14:39--73, 2013.

\bibitem{CDV09}
S.~Cl\'emen\c{c}on, M.~Depecker, and N.~Vayatis.
\newblock Adaptive partitioning schemes for bipartite ranking.

\bibitem{CV09ieee}
S.~Cl\'emen\c{c}on and N.~Vayatis.
\newblock Tree-based ranking methods.
\newblock {\em IEEE Transactions on Information Theory}, 55(9):4316--4336,
  2009.

\bibitem{Clemencon08Ranking}
S.~Cl{\'e}men{\c{c}}on, G.~Lugosi, and N.~Vayatis.
\newblock {Ranking and Empirical Minimization of U-Statistics}.
\newblock {\em The Annals of Statistics}, 36(2):844--874, 2008.

\bibitem{Fawcett06}
T.~Fawcett.
\newblock {An Introduction to ROC Analysis}.
\newblock {\em Letters in Pattern Recognition}, 27(8):861--874, 2006.

\bibitem{Huo2018a}
J.~Huo, Y.~Gao, Y.~Shi, and H.~Yin.
\newblock Cross-modal metric learning for auc optimization.
\newblock {\em IEEE Transactions on Neural Networks and Learning Systems},
  PP(99):1--13, 2018.

\bibitem{Jain2000a}
A.~Jain, L.~Hong, and S.~Pankanti.
\newblock Biometric identification.
\newblock {\em Communications of the ACM}, 43(2):90--98, 2000.

\bibitem{Jain2014a}
A.~K. Jain, A.~Ross, and S.~Prabhakar.
\newblock An introduction to biometric recognition.
\newblock {\em IEEE Transactions on Circuits and Systems for Video Technology},
  14(1):4--20, 2004.

\bibitem{Jain2017a}
L.~Jain, B.~Mason, and R.~Nowak.
\newblock {Learning Low-Dimensional Metrics}.
\newblock In {\em NIPS}, 2017.

\bibitem{Jin2009a}
R.~Jin, S.~Wang, and Y.~Zhou.
\newblock {R}egularized {D}istance {M}etric {L}earning: {T}heory and
  {A}lgorithm.
\newblock In {\em NIPS}, 2009.

\bibitem{Kulis2012a}
B.~Kulis.
\newblock {M}etric {L}earning: {A} {S}urvey.
\newblock {\em {F}oundations and {T}rends in {M}achine {L}earning},
  5(4):287--364, 2012.

\bibitem{Lee90}
A.~J. Lee.
\newblock {\em {${U}$-statistics: Theory and practice}}.
\newblock Marcel Dekker, Inc., New York, 1990.

\bibitem{McFee2010a}
B.~McFee and G.~R.~G. Lanckriet.
\newblock {M}etric {L}earning to {R}ank.
\newblock In {\em ICML}, 2010.

\bibitem{Quinlan}
J.~Quinlan.
\newblock {Induction of Decision Trees}.
\newblock {\em Machine Learning}, 1(1):1--81, 1986.

\bibitem{Verma2015}
N.~Verma and K.~Branson.
\newblock Sample complexity of learning mahalanobis distance metrics.
\newblock In {\em NIPS}, 2015.

\bibitem{VCB18}
R.~Vogel, S.~Cl\'emen\c{c}on, and A.~Bellet.
\newblock {A Probabilistic Theory of Supervised Similarity Learning: Pairwise
  Bipartite Ranking and Pointwise ROC Curve Optimization}.
\newblock In {\em ICML}, 2018.

\bibitem{Weinberger2009}
K.~Q. Weinberger and L.~K. Saul.
\newblock {D}istance {M}etric {L}earning for {L}arge {M}argin {N}earest
  {N}eighbor {C}lassification.
\newblock {\em {J}ournal of {M}achine {L}earning {R}esearch}, 10:207--244,
  2009.

\end{thebibliography}
%

\section{Appendix}
Code is available on the authors' repository.
\footnote{\url{https://github.com/RobinVogel/On-Tree-based-methods-for-Similarity-Learning}}

\subsection{Acknowledgments}

This work was supported by IDEMIA. We thank the LOD reviewers for their
constructive input.

\subsection{Illustrative figures}

\Cref{fig:anom_tree} represents a fully grown tree of depth $3$ with its associated scores.
\Cref{fig:SLFRK} represents a split produced by the LeafRank procedure.

 \begin{figure}[ht]
 \begin{center}
     \centerline{\includegraphics[width=0.7\linewidth]{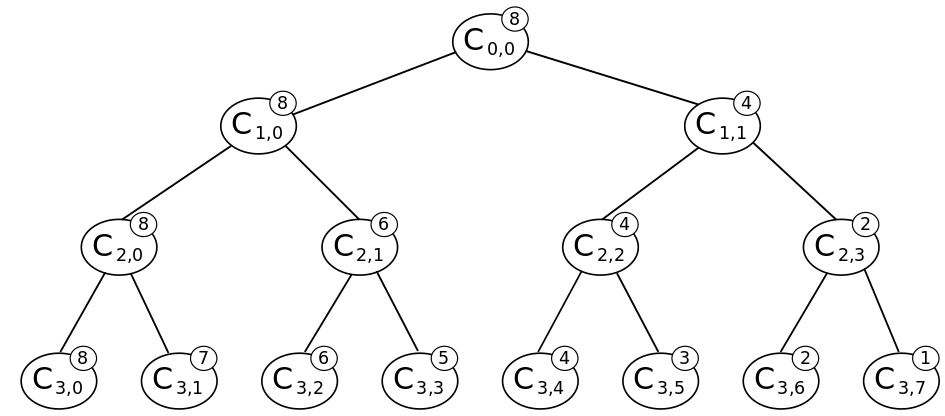}}
 \caption{A piecewise constant similarity function described by an oriented binary subtree $\mathcal{T}$. For any pair $(x,x')\in \mathcal{X}^2$, the similarity score $s_{\mathcal{T}}(x,x')$ can be computed very fast in a top-down manner using the heap structure: starting from the initial value $2^J$ at the root node, at each internal node $\mathcal{C}_{j,k}$, the score remains unchanged if $(x,x')$ moves down to the left sibling and one subtracts $2^{J-(j+1)}$ from it if $(x,x')$ moves down to the right.}
 \label{fig:anom_tree}
 \end{center}
 \end{figure}

 \begin{figure}[ht]
 \vskip -1cm
 \begin{center}
     \centerline{\includegraphics[width=0.8\linewidth]{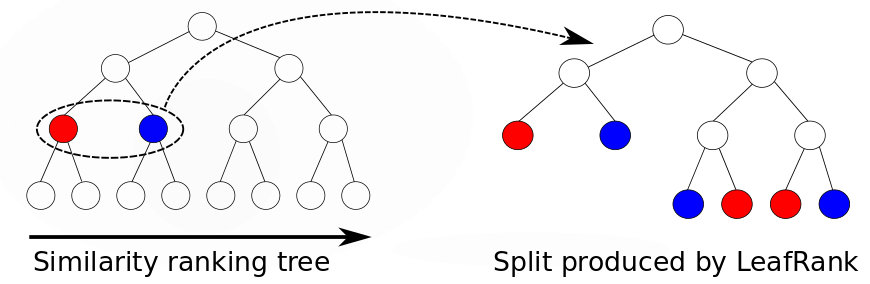}}
     \centerline{\includegraphics[width=0.3\linewidth]{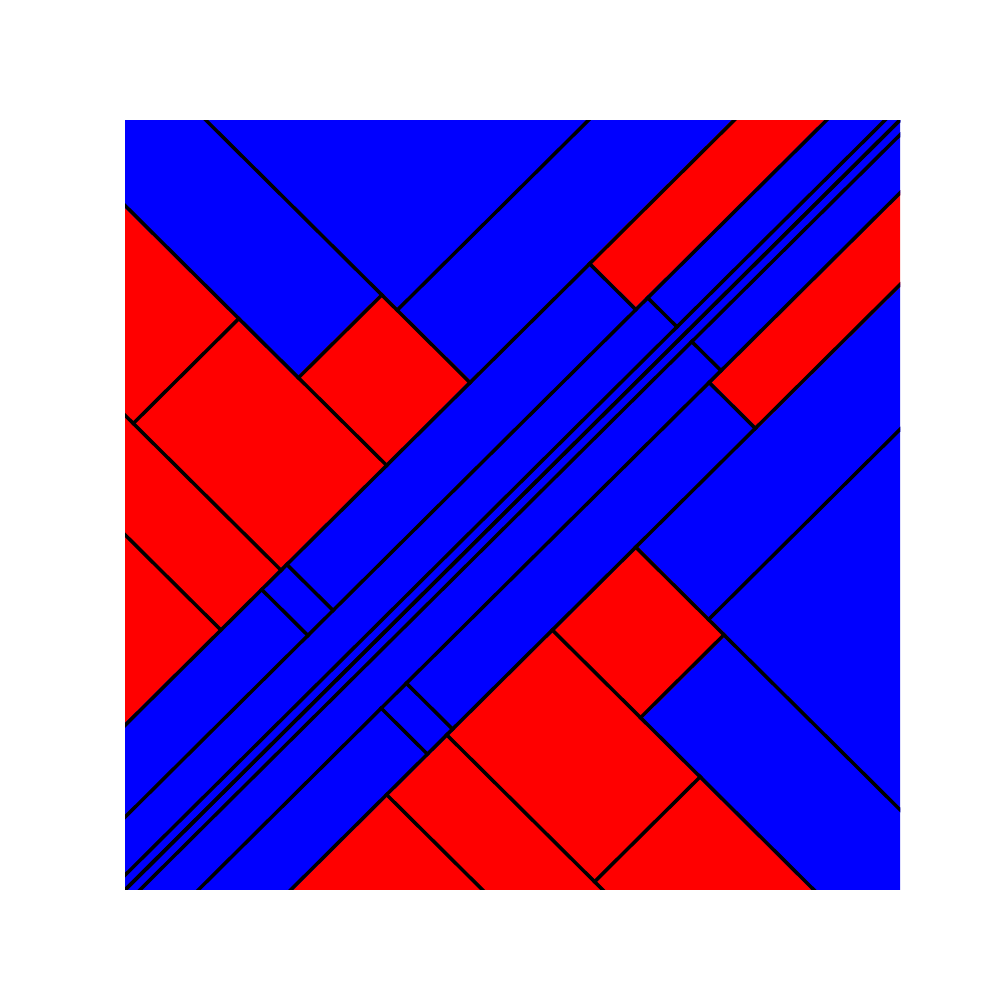}}
 \caption{Symmetric split produced by the {\sc Symmetric LeafRank} procedure.}
 \label{fig:SLFRK}
 \end{center}
 \end{figure}

\subsection{Representation of proposal functions for $\X\times\X = \R\times\R$}

We illustrate visually the outcomes of TreeRank for different 
proposition regions, for a similarity function on the unit square $[0,1]\times [0,1]$.
To obtain a symmetric similarity function, a natural approach is to transform the data using any
function $f : \X\times\X \to \text{Im}(f)$ such that $f(x,x') = f(x',x)$ and then choose a collection of
regions $\mathcal{D} \subset \mathcal{P}(\text{Im}(f))$, to form $\mathcal{C}$ such that 
\begin{align*}
    \mathcal{C} = \left\{ x,x' \in \X \times \X \; \vert \; f(x,x') \in D \right\}_{D \in \mathcal{D}}.
\end{align*}
The $i$-th element of the vector $f(x,x')$ will be written $f^{(i)}(x,x')$.

\begin{figure}
\centering
\begin{subfigure}{.33\textwidth}
  \centering
  \includegraphics[width=\linewidth]{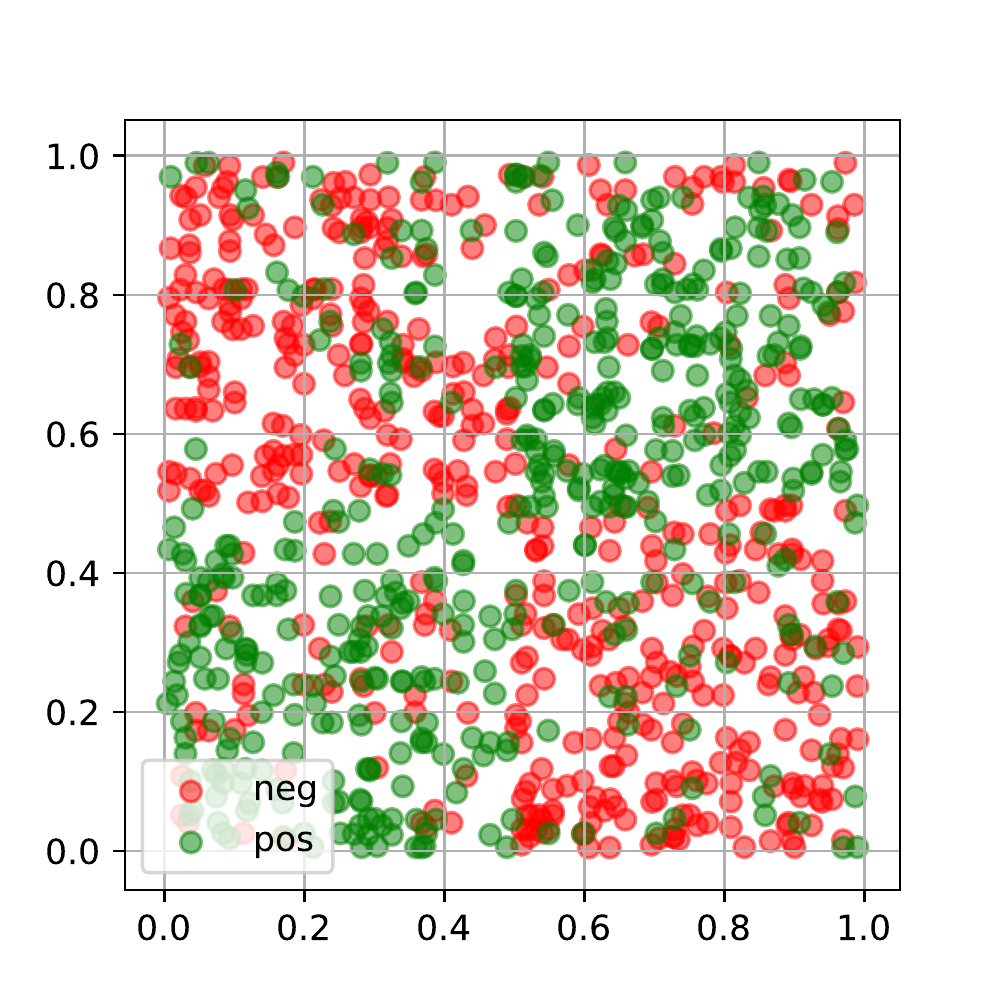}
  \caption{Training pairs}
  \label{fig:pairs}
\end{subfigure}%
\begin{subfigure}{.33\textwidth}
  \centering
  \includegraphics[width=\linewidth]{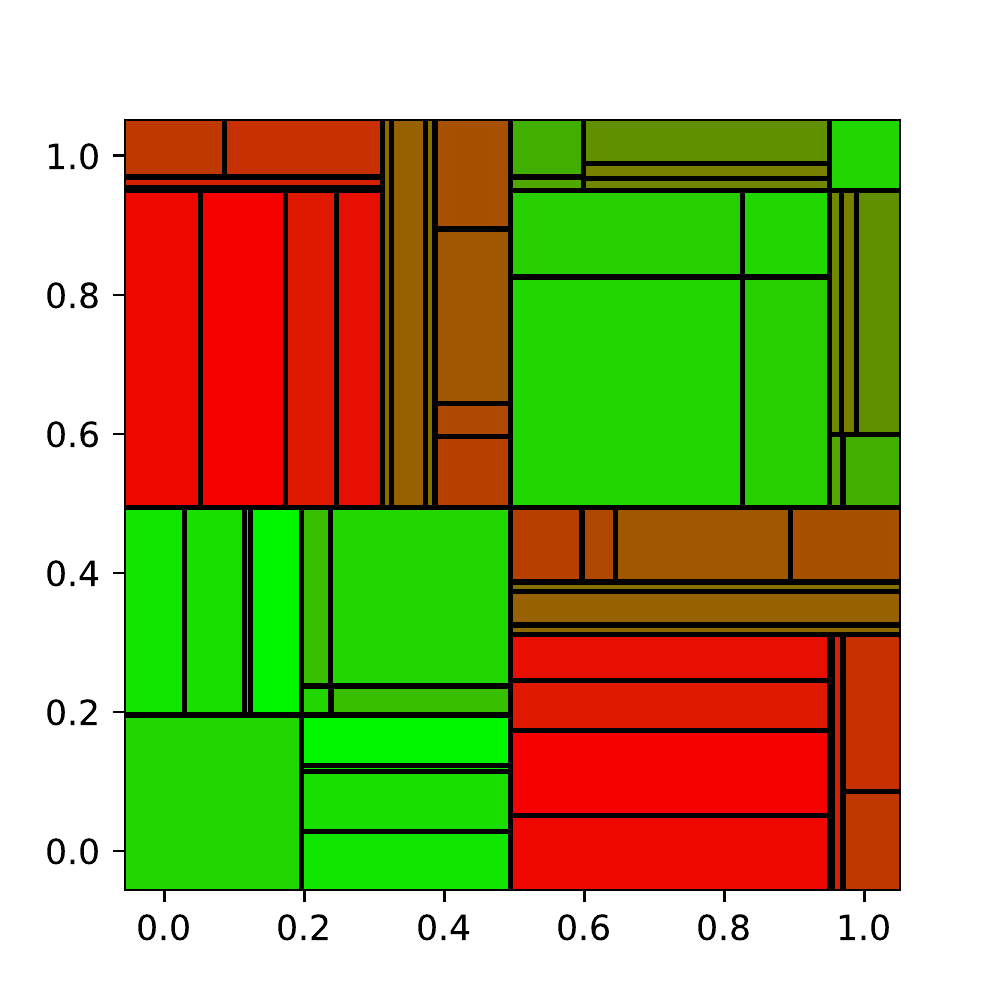}
  \caption{$C_{\text{sq}}$}
  \label{fig:reg_minmax}
\end{subfigure}%
\begin{subfigure}{.33\textwidth}
  \centering
  \includegraphics[width=\linewidth]{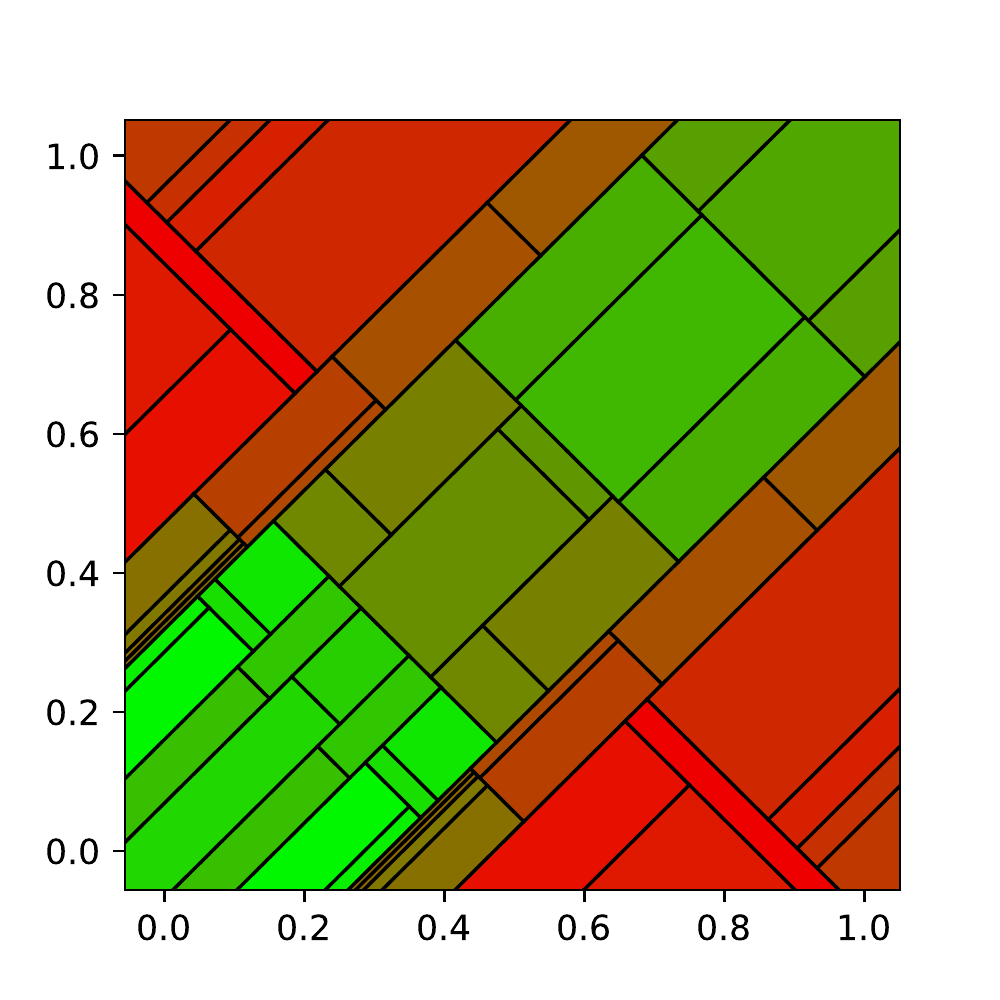}
  \caption{$C_{\text{diag}}$}
  \label{fig:reg_diag}
\end{subfigure}%
\caption{Representation of TreeRank score function with different proposal regions.
The $x$-axis corresponds to $x_1$ while the $y$-axis corresponds to $x_1'$.}
\label{fig:regions_illustration}
\end{figure}

In that context, we present two approaches:
\begin{itemize}
    \item Set $f(x,x') = \binom{x\vee x'}{x\wedge x'}$ where $x\vee x'$ and
	$x\wedge x'$ respectively stand for the element-wise maximum and
	minimum of $x$ and $x'$. We introduce the collection
	$\mathcal{C}_{\text{sq}}$ of all regions:
        \begin{align*}
            \left \{ x,x' \in \X \times \X \Big / 
                \left (\sigma f^{(i)}(x,x') \ge \sigma A \right )\otimes 
            \left (\sigma f^{(i+D)}(x,x') \le \sigma A \right) \right \}
        \end{align*}
	where $i \in \{ 1, \dots, D \}$, $\sigma \in \{-1, +1\}$, $A \in \R$
	and $\otimes$ is the standard XOR.
    \item Set $f(x,x') = \binom{\lvert x-x' \rvert}{x+x'}$. We introduce the collection
	$\mathcal{C}_{diag}$ of all regions:
        \begin{align*}
            \left \{ x,x' \in \X \times \X \Big / \sigma f^{(i)}(x,x') \ge \sigma A \right \}
        \end{align*}
        where $i \in \{ 1, \dots, D \}$, $\sigma \in \{-1, +1\}$, $A \in \R$.
\end{itemize}

We illustrate with \cref{fig:regions_illustration} the results of the outcome
of the TreeRank algorithm with either one of these two approaches, in a simple
case where $\X = [0,1]$, $\mu(x) = 1$, $K=2$ and $\p\{ Y=2 |X=x\} = 0.6 \cdot
\I\{x\ge 0.5 \} + 0.2$.
More complicated decision regions can be chosen, such as any linear decision
function on the transformation $f(x,x')$ of the pair $x,x'$. As stated in \cref{subsection_3-1},
those could be learned for example by an asymmetrically weighted SVM.\\

%
%

\subsection{Details about the synthetic data experiments of \cref{exp_sim_data}}

Assume a fully grown tree $\mathcal{T}$ of depth $D_{\text{gt}}$, with terminal
cells $\mathcal{C}_l \subset \X \times \X$ for all $0 \le l \le
L:=2^{D_{\text{gt}}}-1$. The tree is constructed with splits on the
transformation of the input space $\X\times\X$ by the function $f$ introduced
in \cref{lemma:sym_transform}. The split is chosen by selecting the split
variable uniformly at random, and the split value using a uniform law over that
variable on the current cell. The distribution of the data is assumed to be
defined by $p_+$, $F_+ = \sum_{l=1}^L \delta_l^+ \cdot
\mathcal{U}(\mathcal{C}_l)$ and $F_- = \sum_{l=1}^L \delta_l^- \cdot
\mathcal{U}(\mathcal{C}_l)$ where $\mathcal{U}(\mathcal{C}_l)$ is the uniform
distribution over $\mathcal{C}_l$.  Introduce $\sigma$ as the permutation that
orders the cells $C_l$ by decreasing $\delta_l^+ / \delta_l^-$, i.e.
$\delta_{\sigma(l)}^+ / \delta_{\sigma(l)}^- \ge \delta_{\sigma(l+1)}^+ /
\delta_{\sigma(l+1)}^- $ for all $0 \le l \le L-1$, then the optimal ROC curve
$\roc^*$ is the line that connects the dots $(0,0)$ and $(\sum_{j=0}^l
\delta_{\sigma(j)}^-, \sum_{j=0}^l \delta_{\sigma(j)}^+)$ for all $0\le l \le
L$. 

Now we detail our choice for the specification of the parameters $\delta^+_l$
and $\delta^-_l$.  Assume $\sigma$ to be the identity permutation. To study the
ability of our method to retrieve the optimal ROC curve for different levels of
statistical noise, introduce a noise parameter $0 < \delta < 1$ and fix
$\delta_l^+ = c_{\delta}^+ \cdot \delta^{l/L}$, and $\delta_l^- = 
c_{\delta}^- \cdot \delta^{-l/L}$ for all $0 \le l \le L$, with $c_{\delta}^+$ and
$c_{\delta}^-$ normalization constants in $l$ such that both sets
$\{\delta_l^+\}_{0 \le l \le L}$ and $\{\delta_l^-\}_{0 \le l \le L}$ sum to
one.

When $\delta$ is close to $0$, $\roc^*$ approaches the unit step, 
whereas when $\delta$ is close to $1$, $\roc^*$ approaches the ROC of random assignment. 
The experiments presented here used $\delta = 0.01$, which makes
for an $\auc^*$ of $0.96$ approximately. By varying the parameter $\delta$, one can
study the outcome of our approach for different levels of statistical noise.

The first experiment shows that the learned model $s_D$ generalizes poorly when
positive instances are rare, as shown in the bounds of \cite{VCB18}. The second
one that when $D_n \sim \sqrt{\log n }$, learned models stay decent, as show by
\cref{thm:rate}. The last experiment illustrates the fact that using an overly
deep tree comparatively to the ground truth does not hinder performance, thanks
to the global nature of the ranking problem.


\subsection{Real data experiments}\label{exp_real_data}
We compare the performance of our approach to the widely acclaimed metric
learning technique LMNN, see \cite{Weinberger2009}, as well as a similarity
derived from the cosine similarity of a low-dimensional neural network encoding
of the instances, optimized for classification with a softmax cross-entropy
loss. For that matter, we use the MNIST database with reduced dimensionality by
PCA.  The neural network approach is inspired by state of the art techniques in
applications of similarity learning, such as in facial recognition.  It has
shown outstanding performance, but is not directly derived from the ranking
problem that these systems usually tackle.

The MNIST database of handwritten digits has a training set of 60,000 images and a test set
of 10,000 images and is widely used to benchmark classification algorithms. 
Each image represents a number between 0 and 9 with a monochrome image of $28\times 28$
pixels, which makes for $K=10$ classes and an initial dimensionality of $784$.
The standard principal components analysis (PCA) was set to keep $95\%$
of the explained variance, which reduces the dimensionality of the data to $d=153$. 
This first step was necessary to limit the memory requirements of the LMNN algorithm.
We used the implementation of LMNN provided by the python package \emph{metric-learn},
and changed the regularization parameter to be $0.01$.

The neural network approach learned an encoding $e: \X = \R^d \to \R^{d_e}$ of size
$d_e=128$, used for classification at training time, with a simple
softmax-cross entropy behind a fully connected $d_e\times K$ layer.
The encoding was composed of three stacked fully connected layers followed
by ReLU activations of sizes $153\times 146$, $146\times 140$ and $140 \times 134$,
and finally a $134\times 128$ fully connected layer without an activation function.
These layer sizes are arbitrary and were simply chosen as a linear interpolation
between the input size $d$ and output size $d_e$. The similarity between two instances
is computed using a simple cosine similarity between their embeddings.

Our approach was based off a ranking forest with for symmetric LeafRank an
asymmetric classification tree over the transformed data of fixed depth $5$,
see \cref{fig:SLFRK} for an exemple of this type of proposal region. The
ranking forest aggregates the results of 44 trees of depth 15
learned on only $10^5$ pairs each. Refer to \cite{CDV09} and \cite{CDV13} for
details on ranking forests. ROC curve plots are shown in \cref{fig:roc_models}. 
For now, our method shows higher performance than the linear metric learning approach,
but performs worse than the neural network encoding approach. Further work will aim
to improve the performance of our approach, perhaps with a better LeafRank algorithm.

\begin{figure}
\centering
\includegraphics[width=0.8\linewidth]{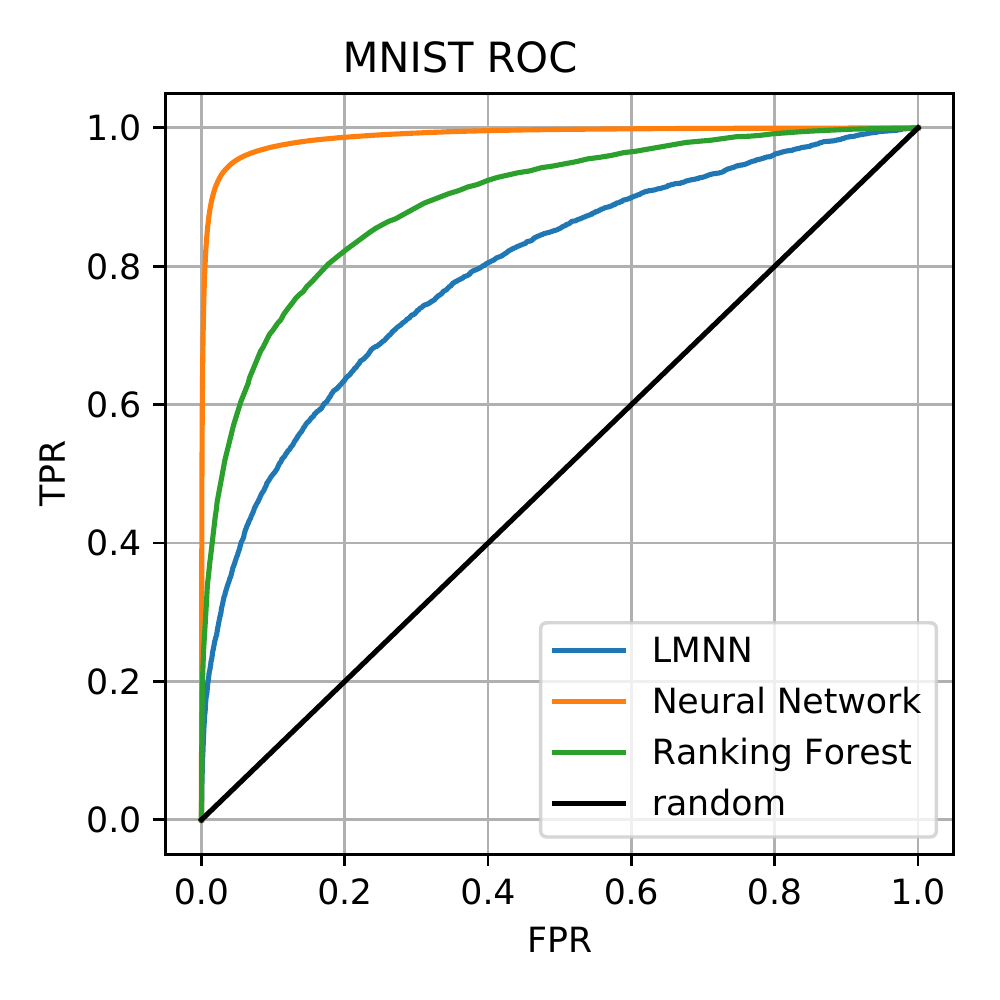}
\caption{ROC curves for the real data experiments.}
\label{fig:roc_models}
\end{figure}

\end{document}